\documentclass{article}
\usepackage{amsmath,graphicx,mlspconf}
\usepackage{xcolor}

\usepackage{amsmath, amssymb,amsthm}
\newtheorem{proposition}{Proposition}
\usepackage{tikz}
\usetikzlibrary{arrows.meta, positioning, fit}
\usetikzlibrary{decorations.pathreplacing}
\usepackage{placeins}
\copyrightnotice{979-8-3195-0884-3/26/\$31.00 {\copyright}2026 IEEE}

\toappear{2026 IEEE International Workshop on Machine Learning for Signal Processing, Sep.\ 28-- Oct.\ 1, 2026, Atlanta, USA}

\title{Time-Warping Estimation via Stationarity-Based Learning \\ of the De-Warped Signal}
\name{%
    Corentin Presvôts%
    \qquad Adrien Meynard%
}
\address{%
    CNRS, ENS de Lyon, LPENSL, UMR 5672, Lyon, France
}

\begin{document}
%\ninept

\maketitle

\begin{abstract}
Time-warping estimation is a fundamental problem in signal processing with applications in bioacoustics, radar, and biomedical analysis. This paper introduces a Time-Warping Estimation Trainable (TWET) model for estimating time-warping functions from a single observation. The proposed approach formulates time-warping estimation as a stationarization problem in the wavelet domain and leverages a hierarchical dilated convolutional architecture to estimate the time-warping functions. A differentiable stationarity criterion is introduced for end-to-end optimization. TWET is compared with existing approaches. Experimental results show improved deformation reconstruction accuracy together with significantly reduced computation time, making the framework compatible with low-latency applications.
\end{abstract}
\begin{keywords}
Time-warping estimation,
wavelet transform,
non-stationary signals,
convolutional neural networks,
inverse problems.
\end{keywords}

\newcommand{\cem}[1]{\textcolor{blue}{cem: #1}}

\section{Introduction}
\label{sec:introduction}

Time-warping phenomena arise in many real-world signals and constitute a major source of non-stationarity. Such deformations result from propagation effects (e.g., Doppler shifts), medium variability, or intrinsic physiological dynamics~\cite{priestley1988non}. They are encountered in a wide range of applications, including bioacoustics, where they affect the temporal evolution of animal vocalizations~\cite{Ioana2004use}, radar and sonar, where Doppler-induced distortions reflect relative motion between emitter and receiver~\cite{Smith2010Radar}, and biomedical signals such as cardiac recordings, where physiological variability requires temporal alignment for reliable analysis~\cite{Tuzcu2005Dynamic}. Estimating these deformations is therefore a key problem, as it enables the recovery of intrinsic stationary structure and improves downstream tasks such as denoising~\cite{souriau2022fetal} or classification~\cite{liu2024novel}.

To address this issue, the observed signal can be interpreted as a time-reparameterized version of a latent stationary process, where local variations along the time axis encode the deformation~\cite{Omer2017Time}. Estimating this time-warping function from a single observation remains a challenging inverse problem, and most existing approaches rely either on multiple realizations or on strong structural assumptions~\cite{chithra2025matched}. In this context, time–scale representations, in particular wavelet transforms, provide a natural framework, as time-warping induces interpretable distortions in the wavelet domain. Under suitable regularity assumptions, these effects can be approximated as local translations along the scale axis~\cite{Omer2017Time}. This property has motivated a variety of estimation strategies, including transport-based methods~\cite{Clerc2003Estimating}, energy distribution approaches~\cite{Omer2017Time}, and maximum-likelihood formulations such as the Joint Estimation of Frequency, Amplitude and Spectrum (JEFAS)~\cite{Meynard2018Spectral,Meynard2023Synthesis}.

Despite these advances, existing approaches still present several limitations. In particular, they may struggle to accurately estimate both slowly and rapidly varying time-warping functions,  their performance depends heavily on carefully tuned hyperparameters, and their computational complexity often prevents their use in low-latency applications.

To address these limitations, a trainable scheme for time-warping estimation, referred to as TWET, is proposed. The method operates on time--scale representations and formulates time-warping estimation as a stationarization problem in the wavelet domain. It relies on a hierarchical dilated convolutional architecture to estimate the time-warping function from a single observation. A key component of the approach is a differentiable stationarity criterion, which provides a task-driven objective function and enables end-to-end optimization. The formulation can be optimized on a signal-by-signal basis, as in existing signal-specific approaches, or trained on representative data to learn a deformation model for signals from the same application domain.

The remainder of this paper is organized as follows. Section~\ref{sec:Framework} formalizes the inverse problem. Section~\ref{sec:Proposed Method} presents the proposed approach. Section~\ref{sec:Results} reports the experimental results, and Section~\ref{sec:Conclusion} concludes the paper.

\section{Framework}
\label{sec:Framework}

This section introduces the time-warping estimation problem. Section~\ref{subsec:Modeling Time-Warped Signals} presents the deformation model, Section~\ref{subsec:Time-Warping as Scale Translations} describes its effect in the time-scale domain, and Section~\ref{subsec:Existing Estimation Methods} reviews existing methods.

\subsection{Modeling Time-Warped Signals}
\label{subsec:Modeling Time-Warped Signals}

Let $y(t)$ denote an observed real-valued continuous-time signal, defined for $t\in\mathbb{R}$. The signal is assumed to result from a time-warped version of an underlying latent signal $x(t)$. The time-warping function, denoted $\gamma(t)$, is assumed to be continuously differentiable and strictly increasing. The resulting model is given by
\begin{equation}
\label{eq:y(t)}
    y(t)=x(\gamma(t)).
\end{equation}
It is further assumed that $x(t)$ is modeled as a realization of a zero-mean, wide-sense stationary stochastic process, \textit{i.e.},
\begin{align}
    \mathbb{E}\left[x\left(t\right)\right]&=0,\quad\forall t,\in\mathbb{R}\\
\mathbb{E}\left[x\left(t\right)x\left(t+\tau\right)\right]&=\mathbb{E}\left[x\left(0\right)x(\tau)\right],\quad\forall t,\tau\in\mathbb{R}.
\end{align}

The stationarity assumption implies that the statistical structure of $x(t)$ is time-invariant, so that any observed nonstationarity observed in $y(t)$ can be attributed solely to the deformation $\gamma(t)$. 

The objective is therefore to estimate $\gamma(t)$ from the sole observation of $y(t)$. Equivalently, this problem can be interpreted as a stationarization task, in which a deformation $\gamma(t)$ is sought such that
\begin{equation}
\label{eq:dewarping}
    x(t)=y(\gamma^{-1}(t)).
\end{equation}

Direct estimation of the deformation $\gamma(t)$ in the time domain is challenging due to the unknown structure of the latent signal $x(t)$. To overcome this difficulty, the problem is reformulated in a time-scale representation, where time-warping can be expressed in a more structured form.

\subsection{Time-Warping as Scale Translations}
\label{subsec:Time-Warping as Scale Translations}

The estimation of time-warped signals can be carried out naturally in a time-scale representation. In particular, it has been shown that smooth time-warping induces a simple and structured transformation of the wavelet coefficients, which makes this domain especially suitable for time-warping function estimation.

To this end, the continuous wavelet transform is considered. Let $\psi(t)\in L^{2}(\mathbb{R})$ denote a mother wavelet. The wavelet coefficients of $y(t)$ are defined as
\begin{equation}
    W_\mathrm{y}(s,u)=q^{-\frac{s}{2}}\int_{\mathbb{R}}y(t)\psi^*\left(\frac{t-u}{q^{s}}\right)\mathrm{d}t, \quad q>1,
\end{equation}    
where $s>0$ denotes the scale parameter and $u\in\mathbb{R}$ the time variable.

Under suitable regularity assumptions on $\gamma(t)$ and on the chosen wavelet,  it has been shown in~\cite{Omer2017Time} that the wavelet coefficients of $y(t)$ satisfy the following approximation.

\begin{proposition}
\label{prop:approx}
Assume that $x(t)$ is a wide-sense stationary process and that $\gamma(t)\in\mathcal{C}^{1}$ with slowly varying derivative. Then, for sufficiently regular wavelets, the wavelet transforms of $x(t)$ and $y(t)=x(\gamma(t))$ are related by
\begin{equation}
    \label{eq:approx}
    W_\mathrm{y}(s,u)\approx W_{x}(s+\delta(u),\gamma(u)),
\end{equation}
where
\begin{equation}
    \delta(u)=\log_{q}\big(\gamma'(u)\big).
\end{equation}
The approximation error in~\eqref{eq:approx} is controlled, among other factors, by $\|\gamma''\|_\infty = \sup_t |\gamma''(t)|$.
\end{proposition}

\begin{proof}
Given in~\cite{Omer2017Time}.
\end{proof}

This result plays a central role in the analysis of nonstationary signals. Since the latent signal $x(t)$ is assumed to be stationary, its wavelet transform is statistically invariant with respect to time. In contrast, time-warping manifests itself in the wavelet domain as a translation along the scale axis. Consequently, estimating and compensating for $\delta(u)$ allows the restoration of stationarity in the time-scale representation. This property directly motivates a large class of estimation methods that aim at recovering the deformation from observed wavelet coefficients.

\subsection{Existing Estimation Methods}
\label{subsec:Existing Estimation Methods}

Several methods in the literature exploit the approximation result of Proposition~\ref{prop:approx} to estimate the deformation function $\delta(u)$ from the wavelet transform. Once $\delta(u)$ is estimated, the derivative of the time-warping function is recovered as $\gamma'(u)=q^{\delta(u)}$, and $\gamma(u)$ is obtained, up to an additive constant, by integration.

In~\cite{Clerc2003Estimating}, the authors exploit a local relationship between the partial derivatives of the scalogram (\textit{i.e.}, the squared modulus of the wavelet transform) to estimate the deformation. This leads to the estimator
\begin{equation}
    \widehat{\delta}'\left(u\right)=\lim_{s\rightarrow 0}\left(\frac{\partial_{u}\left|W_\mathrm{y}(s,u)\right|^{2}}{\partial_{s}\left|W_\mathrm{y}(s,u)\right|^{2}}\right).
\end{equation}
In practice, this limit is approximated using small scales. Nevertheless, the lack of signal energy at fine scales and the sensitivity of numerical differentiation to noise significantly limit the robustness of this approach.

An alternative approach, which is referred to as the Wavelet Scale Barycenter (WSB) method~\cite{Omer2017Time}, relies on a barycentric aggregation across scales
\begin{equation}
    \widehat{\delta}\left(u\right)=\frac{\int_{\mathbb{R}_+}q^{s}\left|W_\mathrm{y}(s,u)\right|^{2}\mathrm{d}s}{\int_{\mathbb{R}_+}\left|W_\mathrm{y}(s,u)\right|^{2}\mathrm{d}s}.
\end{equation}
In contrast to derivative-based methods, WSB avoids explicit differentiation but does not explicitly account for temporal consistency of the wavelet energy distribution, which typically leads to high estimation variance.

More recently, maximum likelihood formulations have been introduced within the JEFAS framework~\cite{Meynard2018Spectral}. Based on the approximation in~\eqref{eq:approx}, the likelihood of the wavelet coefficients is expressed in terms of both the deformation $\delta(u)$ and the power spectrum of the latent stationary signal $x(t)$. Estimation is performed iteratively by alternating between updating $\delta(u)$ and the spectral parameters.

JEFAS relies strongly on the validity of the approximation in~\eqref{eq:approx}, which may deteriorate when $\gamma'(t)$ varies significantly within the support of the wavelet. To improve robustness to fast variations, JEFAS-S~\cite{Meynard2023Synthesis} adopts a Bayesian formulation. Nevertheless, its computational complexity becomes prohibitive for long signals (typically beyond 1024 samples), limiting its applicability in practice.

\section{Proposed Method}
\label{sec:Proposed Method}

A trainable scheme for time-warping estimation, referred to as TWET, is proposed to address limitations of existing approaches, including sensitivity to noise and rapidly varying deformations, as well as computational cost.

Section~\ref{subsec:Discretization and Segmentation} describes the signal discretization and segmentation strategy, Section~\ref{subsec:Inverse Problem} formulates the inverse problem in the time-scale domain, and Section~\ref{subsec:Proposed Time-Warping Estimation Trainable Model} presents the proposed neural architecture.

\subsection{Discretization and Segmentation}
\label{subsec:Discretization and Segmentation}

The observed signal $y(t)$ is sampled uniformly with sampling period $T_{\mathrm{s}} = 1/F_{\mathrm{s}}$, yielding $y_n = y(nT_{\mathrm{s}})$, where $n \in \mathbb{Z}$.

Time-warping estimation is performed on short-time segments of fixed length $N$ to reduce latency and computational cost. The signal is thus partitioned as
\begin{equation}
\boldsymbol{y}_{i} = \left( y_{iN}, \dots, y_{(i+1)N-1} \right)^{T}, \quad i \in \mathbb{Z}.
\end{equation}
The parameter $N$ controls the trade-off between temporal context and latency. The segment index is omitted in the remainder of the paper when no ambiguity arises.

The time-scale representation is computed using a discretized wavelet transform on a linear scale grid $s_m$, $m = 0,\dots,M-1$, and a uniform time grid $u_n = nT_{\mathrm{s}}$, $n = 0,\dots,N-1$, with $q>1$. The coefficients are defined as
\begin{equation}
W_{\mathrm{y},m,n} = W_\mathrm{y}(s_m,u_n),
\end{equation}
forming a matrix $\boldsymbol{W}_\mathrm{y} \in \mathbb{C}^{M \times N}$.

The same discrete notation is used for all estimated and reconstructed quantities, including the deformation derivative $\boldsymbol{\gamma}'$, the scale-shift parameter $\boldsymbol{\delta}$, the reconstructed signal $\boldsymbol{x}$, and its associated wavelet transform $\boldsymbol{W}_\mathrm{x}$.

\subsection{Inverse Problem Formulation}
\label{subsec:Inverse Problem}

The objective is to estimate a deformation operator acting on the wavelet representation of the observed signal such that the resulting representation matches that of an underlying stationary process. According to Proposition~\ref{prop:approx}, time-warping induces approximate translations along the scale axis in the time-scale domain.

Since the values of $\boldsymbol{\delta}$ are not restricted to integers, a direct shift of discrete wavelet coefficients is not well-defined. The de-warping operator is therefore defined through interpolation along the scale axis. For a given deformation $\boldsymbol{\delta}$, the de-warped wavelet coefficient at location $(m,n)$ is defined as
\begin{equation}
\label{eq:W_y_delta}
W^{(\boldsymbol{\delta})}_{\mathrm{y},m,n}
=
\mathcal{I}_{s}\!\left( W_{\mathrm{y}}(\cdot,u_n) \right)
\Big|_{s = s_m-\delta_n},
\end{equation}
where $\mathcal{I}_{s}$ denotes an interpolation operator along the scale axis. The wavelet representation is evaluated on the linear scale grid indexed by $m = 0,\dots,M-1$ and on the discrete time samples $n = 0,\dots,N-1$. The resulting de-warped wavelet representation is denoted by $\boldsymbol{W}^{(\boldsymbol{\delta})}_{\mathrm{y}}$.

Under correct compensation of the deformation, Proposition~\ref{prop:approx} implies that
\begin{equation}
    \boldsymbol{W}_\mathrm{x} \approx \boldsymbol{W}^{(\boldsymbol{\delta})}_\mathrm{y},
\end{equation}
so that the de-warped coefficients are expected to exhibit stationarity along the time axis.

Following~\cite{borgnat2010Testing}, a differentiable stationarity criterion is defined as the average temporal variance $\mathrm{Var}_m(\cdot)$ of the de-warped coefficients along the time index $n$ at each scale $m$
\begin{equation}
\label{eq:stat_score}
    \mathcal{T}(\boldsymbol{W}^{(\boldsymbol{\delta})}_\mathrm{y})
    =
    \frac{1}{M}
    \sum_{m=0}^{M-1}
    \mathrm{Var}_{m}\!\left( \boldsymbol{W}^{(\boldsymbol{\delta})}_\mathrm{y}\right).
\end{equation}
This criterion measures the temporal variability of the wavelet coefficients at each scale. For a Wide-Sense Stationary (WSS) process, their second-order statistics are time-invariant. Therefore, \eqref{eq:stat_score} uses their temporal variance as a tractable surrogate for stationarity, rather than explicitly enforcing WSS or estimating the full autocorrelation function. The objective is to estimate the time-warping function that minimizes this variability. Only the modulus of the wavelet coefficients is retained, as motivated by Proposition~\ref{prop:approx}, which provides an explicit relation for the wavelet modulus under time warping.

The deformation estimation is then formulated as the following regularized inverse problem
\begin{equation}
    \mathcal{L}(\boldsymbol{\delta})
    =
    \mathcal{T}(\boldsymbol{W}^{(\boldsymbol{\delta})}_\mathrm{y})
    +
    \lambda \left\| \nabla \boldsymbol{\delta} \right\|_2^2,
    \quad \lambda > 0,
\end{equation}
where $\nabla \boldsymbol{\delta}$ denotes the discrete temporal gradient of $\boldsymbol{\delta}$. The regularization term enforces temporal smoothness of the estimated deformation, consistent with the assumption that $\gamma'(t)$ varies slowly over time.

The optimal deformation is defined as
\begin{equation}
\label{eq:opt_problem}
    \widehat{\boldsymbol{\delta}}
    =
    \arg\min_{\boldsymbol{\delta}} \mathcal{L}(\boldsymbol{\delta}).
\end{equation}
Nevertheless, solving~\eqref{eq:opt_problem} is challenging due to the strong non-linearity of the warping operator and the ill-posed nature of the problem in the presence of noise. In addition, the resulting optimization landscape is highly non-convex, which makes classical numerical solvers sensitive to initialization and prone to suboptimal local minima.

These limitations motivate a reformulation of the estimation problem within a parametric and data-driven framework, where the deformation is predicted by a trainable model.

\subsection{Proposed Time-Warping Estimation Trainable Model}
\label{subsec:Proposed Time-Warping Estimation Trainable Model}

To improve numerical stability and reduce sensitivity to amplitude variations, a channel-wise normalization is applied to the modulus of the wavelet coefficients $| \boldsymbol{W}_{\mathrm{y}} |$. The input is standardized by subtracting its mean and dividing by its standard deviation, yielding the normalized representation $\boldsymbol{X}^{(0)}$.

\begin{figure}[h!]
\centering
\begin{tikzpicture}
    % Image
    \node[anchor=south west, inner sep=0] (img) at (0,0)
    {\includegraphics[width=\columnwidth]{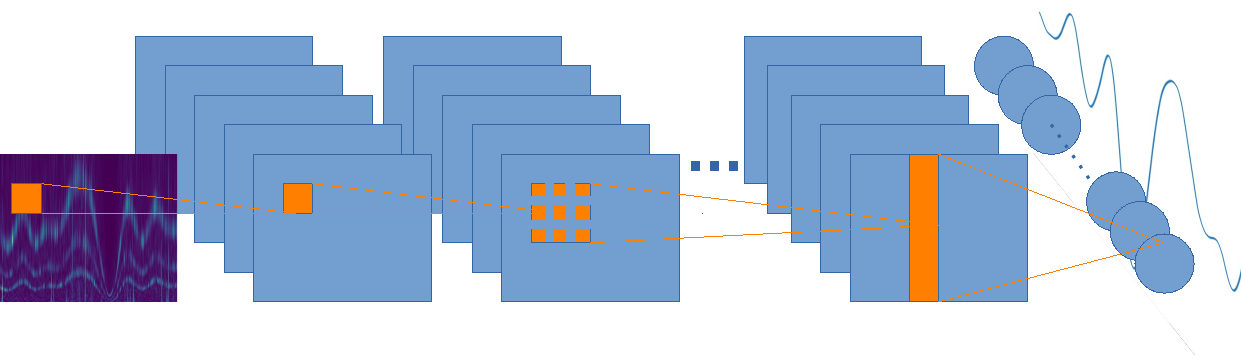}};

    % Coordinate system
    \begin{scope}[x={(img.south east)}, y={(img.north west)}]

        \node at (0.5cm,2.5cm) {$\boldsymbol{X}^{(0)}$};
        \node at (1.7cm,2.5cm) {$\boldsymbol{X}^{(1)}$};
        \node at (3.5cm,2.5cm) {$\boldsymbol{X}^{(2)}$};
        
        \node at (4.8cm,2.5cm) {$\cdots$};

        \node at (6cm,2.5cm) {$\boldsymbol{X}^{(L)}$};

        \node at (7.6cm,2.455cm) {$\boldsymbol{\delta}(\boldsymbol{\Phi)}$};

        \draw[decorate, decoration={brace, amplitude=5pt, mirror}]
            (1.5cm,0.2cm) -- (5cm,0.2cm)
            node[midway,yshift=-0.6cm,align=center] {Dilated convolutions \\ feature extraction};
        \draw[decorate, decoration={brace, amplitude=5pt, mirror}]
            (5.5cm,0.2cm) -- (8cm,0.2cm)
            node[midway,yshift=-0.6cm,align=center] {Convolutional projection};        
    \end{scope}

\end{tikzpicture}
\caption{Architectural overview of the proposed Time-Warping Estimation Trainable (TWET) model for estimating the time-warping derivative $\boldsymbol{\gamma}'$.}
\label{fig:schema_bloc}
\end{figure}
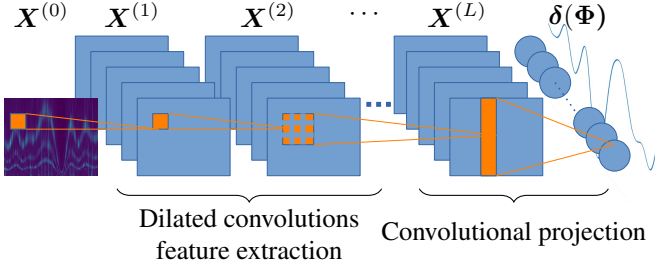

Fig.~\ref{fig:schema_bloc} illustrates the proposed hierarchical convolutional architecture for estimating the deformation from the time-scale representation. The model processes $\boldsymbol{X}^{(0)} \in \mathbb{R}^{C \times M \times N}$ through a sequence of $L$ convolutional blocks.

At layer $\ell = 0,\dots,L-1$, the feature map $\boldsymbol{X}^{(\ell)} \in \mathbb{R}^{C \times M \times N}$ is processed using $C$ dilated 2D convolutions with kernel size $(k_m,k_n)$, where $k_m$ and $k_n$ are odd integers, chosen small compared to $M$ and $N$, respectively, and a dilation factor $2^{\ell}$. The output of the $c$-th convolution is denoted by $\boldsymbol{Z}^{(\ell)}_{c}$, with components given by
\begin{equation}
Z^{(\ell)}_{c,m,n}
=
b^{(\ell)}_{c}
+
\sum_{k=1}^{C}
\sum_{(i,j)\in \mathcal{K}}
w^{(\ell)}_{c,k,i,j}
\, X^{(\ell)}_{k, m + 2^{\ell} i, n + 2^{\ell} j },
\end{equation}
where
$
\mathcal{K}
=
\left\{
-\frac{k_m-1}{2},\dots,\frac{k_m-1}{2}
\right\}
\times
\left\{
-\frac{k_n-1}{2},\dots,\frac{k_n-1}{2}
\right\},
$
and where $w^{(\ell)}_{c,k,i,j}$ and $b^{(\ell)}_{c}$
denote the learnable convolution weights and biases. The $C$ feature maps are concatenated as
\begin{equation}
\boldsymbol{Z}^{(\ell)} =
\left(\boldsymbol{Z}^{(\ell)}_{1}, \dots, \boldsymbol{Z}^{(\ell)}_{C}\right),
\end{equation}
and the output of the $\ell+1$-th layer with $\ell=0,\dots,L-1$, is given by the embedded in a residual structure
\begin{equation}
\boldsymbol{X}^{(\ell+1)} =
\boldsymbol{X}^{(\ell)} +
\phi\!\left(\mathrm{LN}^{(\ell)}(\boldsymbol{Z}^{(\ell)})\right),
\end{equation}
where $\phi(\cdot)$ is a nonlinear activation function, namely the Gaussian Error Linear Unit (GELU)~\cite{hendrycks2016gaussian}, and $\mathrm{LN}^{(\ell)}$ denotes layer normalization~\cite{ba2016layer}.

After $L$ layers, a final convolution aggregates information across scales with kernel $(M,1)$ and dilation factor one
\begin{equation}
\boldsymbol{\delta}(\boldsymbol{\Phi})
=
\mathcal{C}\!\left(\boldsymbol{X}^{(L)}, M, 1, 1\right)
\in \mathbb{R}^{N},
\end{equation}
where $\boldsymbol{\Phi}$ denotes the trainable parameters (see Fig.~\ref{fig:schema_bloc}).

This output corresponds to an estimate of the log-derivative of the deformation, denoted $\log_q \gamma'(t)$ in discrete form. To ensure identifiability and avoid trivial ambiguities such as global shifts, the output is constrained to have zero mean.

Given the estimated deformation, the de-warping operator defined in~\eqref{eq:W_y_delta} is applied to reconstruct an approximation of the latent stationary representation $\boldsymbol{W}_\mathrm{x}$. The network parameters are then learned by minimizing 
\begin{equation}
    \widehat{\boldsymbol{\Phi}}
    =
    \arg\min_{\boldsymbol{\Phi}}
    \mathcal{L}\left(\boldsymbol{\delta}(\boldsymbol{\Phi})\right),
\end{equation}
where $\mathcal{L}(\cdot)$ is defined in Section~\ref{subsec:Inverse Problem}.

\section{Results}
\label{sec:Results}

This section evaluates the proposed TWET framework\footnote{Code: https://github.com/CorentinPresvots/time-warping-estimation}.
TWET is evaluated on a synthetic dataset and compared with two reference approaches: WSB~\cite{Omer2017Time} and JEFAS~\cite{Meynard2018Spectral}\footnote{Code available at https://github.com/AdMeynard/JEFAS} described in Section~\ref{subsec:Synthetic Experiments}. The method of Clerc and Mallat~\cite{Clerc2003Estimating} is not considered due to its limited robustness to noise, while JEFAS-S~\cite{Meynard2023Synthesis} is excluded because of its high computational complexity for long signals. Since the ground-truth deformation derivative $\boldsymbol{\gamma}'$ is available, performance is evaluated by comparing the estimated deformation derivative $\widehat{\boldsymbol{\gamma}}'$ with $\boldsymbol{\gamma}'$.

%The methods are then evaluated on real-world signals in Section~\ref{subsec:Real-World Experiments}, where only the observation $\boldsymbol{y}$ is available. In this setting, performance is assessed through the stationarity of the estimated de-warped signal $\widehat{\boldsymbol{x}}$, obtained by applying the inverse warping operator associated with the estimated deformation and evaluated using~\eqref{eq:stat_score}.

The methods are then evaluated on real-world signals in Section~\ref{subsec:Real-World Experiments}, where only $\boldsymbol{y}$ is available and the framework is optimized on a signal-specific basis. Performance is assessed through the stationarity of the estimated de-warped signal $\widehat{\boldsymbol{x}}$ using~\eqref{eq:stat_score}; with representative training data, the same formulation can instead learn a deformation model for signals from the same application domain.

\subsection{Synthetic Experiments}
\label{subsec:Synthetic Experiments}

For TWET training, $500$ synthetic signals are generated for training, $50$ for validation, and $50$ for testing. All methods are evaluated on the same $50$ test signals. For the proposed architecture, the convolution kernels are set to $k_m = k_n = 3$, the number of convolutional layers is fixed to $L = 4$, and the number of filters per layer is set to $C = 16$.

Each signal is generated from a stationary process $\boldsymbol{x}$ and a strictly positive time-warping derivative $\boldsymbol{\gamma}'$. The deformation is obtained by temporally smoothing Gaussian white noise using a moving-average filter of random length $F_T \in [1,100]$, followed by a shift and rescaling enforcing unit mean, strict positivity, and standard deviation $\sigma_\gamma \in [0,0.5]$.

The latent stationary signal $\boldsymbol{x}$ is generated as a sum of Gaussian narrowband processes obtained by filtering white noise with Hann-shaped spectra of random bandwidths and center frequencies uniformly distributed in $[0,F_s/3]$. This construction produces signals exhibiting both spectral and amplitude modulations. The observed signal $\boldsymbol{y}$ is then obtained using the time-warping operator defined in~\eqref{eq:y(t)}. 

All signals are sampled at $F_s=44100$~Hz and segmented into $N=44100$-sample (1~s) windows. The deformation is estimated locally, with $N$ controlling the trade-off between temporal locality and latency: smaller $N$ improves locality but may miss the full deformation trajectory, whereas larger $N$ provides a more global estimate at the cost of increased latency.

Table~\ref{tab:comparison_methods} reports the Mean Squared Error (MSE) and Mean Absolute Error (MAE) between the estimated deformation derivative $\widehat{\boldsymbol{\gamma}}'$ and the ground truth $\boldsymbol{\gamma}'$, together with two stationarity indices
\begin{equation}
\label{eq:r_T}
r_T=100
\mathcal{T}\left(\boldsymbol{W}^{(\boldsymbol{\delta})}_\mathrm{y}\right),
\qquad
r_J =\mathcal{J}\left(\boldsymbol{W}^{(\boldsymbol{\delta})}_\mathrm{y}\right),
\end{equation}
where $\mathcal{T}(\cdot)$ is a stationary metrics introduced in~\eqref{eq:stat_score} and $\mathcal{J}(\cdot)$ is the negative log-likelihood  introduced in~\cite{Meynard2018Spectral}. Lower values indicate improved stationarity after de-warping. Average runtimes over the $50$ test signals are also reported in Table~\ref{tab:comparison_methods}. 

\begin{table}[h!]
\centering
\caption{Performance over 50 synthetic signals: deformation error, stationarity indices, and runtime.}
\label{tab:comparison_methods}
\begin{tabular}{lccc}
\hline
Method & TWET & JEFAS \cite{Meynard2018Spectral} & WSB \cite{Omer2017Time} \\
\hline
MSE $(\boldsymbol{\gamma}'-\widehat{\boldsymbol{\gamma}}')$ & 0.0070 & 0.0424 & 0.0564 \\
MAE $(\boldsymbol{\gamma}'-\widehat{\boldsymbol{\gamma}}')$ & 0.0525 & 0.1076 & 0.1938 \\
$r_T$ \eqref{eq:r_T} & 4.279 & 6.460 & 7.111 \\
$r_J$ \eqref{eq:r_T}  & -308.9 & -341.5 & -286.2 \\
Time per signal (s) & 0.0008 & 13.5 & 1.7 \\
\hline
\end{tabular}
\end{table}

TWET achieves the best MSE and MAE performance for the reconstruction of the deformation derivative $\widehat{\boldsymbol{\gamma}}'$ compared with existing approaches. Part of this improvement may be explained by the similarity between the training and test distributions. In contrast, JEFAS yields lower average reconstruction accuracy, as the method may fail to converge when the deformation derivative $\boldsymbol{\gamma}'$ varies rapidly. Nevertheless, JEFAS achieves better performance according to the stationarity criterion $r_J$. This behavior can be explained by the fact that JEFAS is explicitly designed to minimize the criterion $r_J$. TWET also provides a substantial computational advantage, reaching real-time compatibility with an average runtime of $0.0008$~s per signal. Although TWET results are obtained using batched GPU processing, both JEFAS and WSB remain significantly more computationally expensive.

%TWET achieves the best performance across on MSE and MAE sur la qualité de reconstruction de $\widehat{\boldsymbol{\gamma}}'$ compared to literature. Part of this improvement can be explained by the similarity between the training and test distributions. In contrast, JEFAS yields lower average performance because it may fail to converge when the deformation derivative $\boldsymbol{\gamma}'$ varies rapidly. On remarque néanmoins que sur le critère $r_T$ JEFAS méthode donne de meilleurs performances, cela est dû au fait que JEFAS est optimisé pour minimiser $r_J$. TWET also provides a substantial computational advantage, reaching real-time compatibility with an average runtime of $0.0008$~s per signal. Although TWET results are obtained using batched GPU processing, both JEFAS and WSB remain significantly more computationally expensive.

\subsection{Real-World Experiments}
\label{subsec:Real-World Experiments}

The proposed method is evaluated on real-world audio signals, including the sound of a car accelerating with gear shifts (Fig.~\ref{fig:real_examples0}), a recording of singing (Fig.~\ref{fig:real_examples1}), and the sound of wind blowing through a window (Fig.~\ref{fig:real_examples2}).

Since no ground-truth deformation is available, the comparison relies on qualitative analysis together with the stationarity measures defined in~\eqref{eq:r_T}. TWET is trained directly on each signal, as the amount of available real-world data is currently insufficient to train a fully generalizable model. In this setting, the proposed framework can therefore be interpreted as a signal-adaptive parametric estimator.

Fig.~\ref{fig:real_examples0}--\ref{fig:real_examples2} illustrate examples. Each figure shows the wavelet transform of the observed signal (top left), the TWET de-warped representation (top right), and the deformation derivatives estimated by WSB, JEFAS, and TWET (bottom).

\begin{figure}[!ht]
\centering
\includegraphics[height=0.36\textwidth,width=0.49\textwidth]{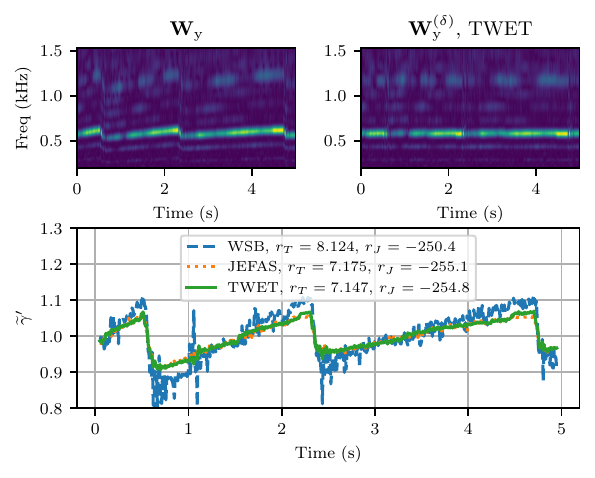}
\vspace{-1.0cm}
\caption{Car acceleration. Top left: wavelet transform $\boldsymbol{W}_\mathrm{y}$ of the observed signal $\boldsymbol{y}$. Top right: estimated de-warped representation obtained with TWET. Bottom: estimated deformation derivatives obtained with WSB, JEFAS, and TWET.}
\label{fig:real_examples0}
\end{figure}

\begin{figure}[!ht]
\centering
\includegraphics[height=0.36\textwidth,width=0.49\textwidth]{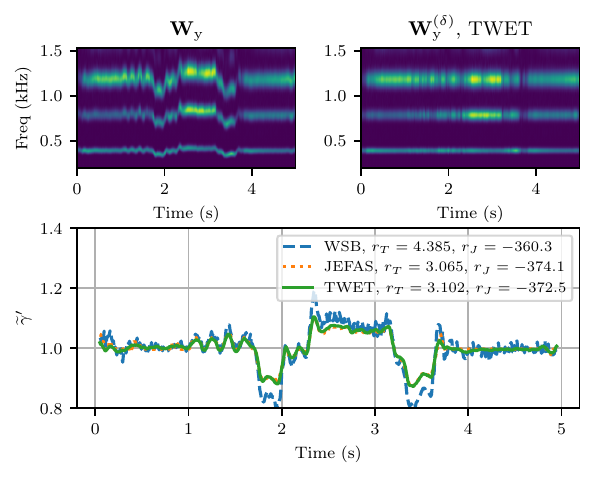}
\vspace{-1.0cm}
\caption{Singing voice recording: same layout as Fig.~\ref{fig:real_examples0}.}
\label{fig:real_examples1}
\end{figure}

\begin{figure}[!ht]
\centering
\includegraphics[height=0.36\textwidth,width=0.49\textwidth]{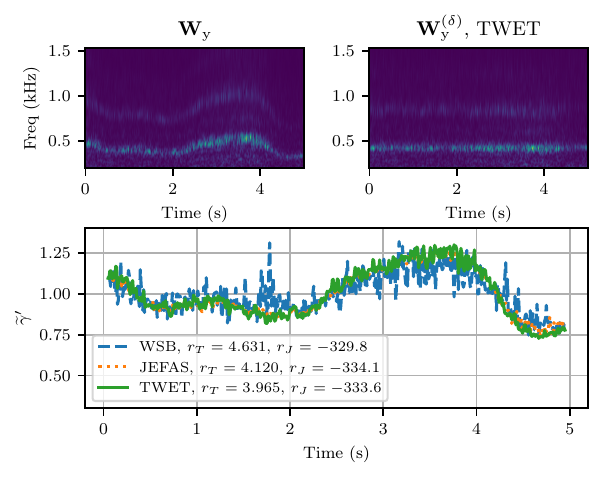}
\vspace{-1.0cm}
\caption{Wind blowing: same layout as Fig.~\ref{fig:real_examples0}.}
\label{fig:real_examples2}
\end{figure}

Overall, TWET slightly outperforms JEFAS and significantly improves over WSB. In particular, WSB appears more sensitive to noise and amplitude variations, resulting in considerably noisier deformation estimates. %On the examples of Fig.~\ref{fig:real_examples0}--\ref{fig:real_examples2}, the criteria $r_T$ and $r_J$ may favor different methods, which is expected since each approach is optimized according to its own objective function: TWET minimizes~\eqref{eq:stat_score}, whereas JEFAS relies on the likelihood formulation introduced in~\cite{Meynard2018Spectral}.
Although TWET demonstrates strong empirical performance, its generalization capability remains dependent on the diversity and representativeness of the training data.

\section{Conclusion}
\label{sec:Conclusion}

This paper introduced TWET, a trainable scheme for time-warping estimation based on time-scale representations and a differentiable stationarity criterion. The problem is formulated as an inverse problem in the wavelet domain and solved using a hierarchical dilated convolutional architecture that estimates time-warping functions from a single observation.

Experimental results on synthetic signals demonstrated that TWET achieves lower reconstruction errors and improved stationarity restoration than reference methods. In particular, the proposed framework remains effective for rapidly varying time-warping deformations, a regime where methods such as JEFAS tend to fail or produce unstable estimates. In addition, TWET significantly reduces computation time for deformation estimation, making it compatible with low-latency applications. Experiments on real-world signals further highlighted the robustness of the proposed method to noise and amplitude variations.

Future work includes joint estimation of amplitude modulation and time deformation, as well as incorporating multiple wavelet resolutions to better handle deformation dynamics ranging from slow to rapid variations.

\bibliographystyle{IEEEbib}
\bibliography{strings,refs}

\end{document}